\documentclass[english]{article}

\usepackage[T1]{fontenc}
\usepackage[utf8]{inputenc}
\usepackage{amsmath}
\usepackage{amsthm}
\usepackage{amssymb}
\usepackage{graphicx}
\usepackage{enumerate}
\usepackage{algorithmic}
\usepackage[]{algorithm2e}
\usepackage{wrapfig}

\makeatletter
\theoremstyle{plain}
\newtheorem{thm}{\protect\theoremname}
\theoremstyle{plain}
\newtheorem{lem}[thm]{\protect\lemmaname}
\theoremstyle{plain}

\usepackage{microtype}
\usepackage{graphicx}
\usepackage{xcolor}         
\usepackage{booktabs} 
\usepackage{url}            

\usepackage{hyperref}

\usepackage[preprint]{neurips_2022}

\newcommand{\E}{\mathbb{E}}
\DeclareMathOperator*{\minimize}{minimize}

\@ifundefined{showcaptionsetup}{}{%
 \PassOptionsToPackage{caption=false}{subfig}}
\usepackage{subfig}
\makeatother

\usepackage{babel}
\providecommand{\lemmaname}{Lemma}
\providecommand{\theoremname}{Theorem}
\providecommand{\corrname}{Corollary}

\title{Grad-GradaGrad? A Non-Monotone Adaptive Stochastic Gradient Method}

\author{%
  Aaron Defazio\\
  Meta AI\\
   \And
   Baoyu Zhou\\
   Lehigh University\\
   \And
   Lin Xiao\\
  Meta AI\\
}

\begin{document}

\maketitle

\begin{abstract}
The classical AdaGrad method adapts the learning rate by dividing
by the square root of a sum of squared gradients. Because this sum
on the denominator is increasing, the method can only decrease step
sizes over time, and requires a learning rate scaling hyper-parameter to
be carefully tuned. To overcome this restriction, we introduce GradaGrad,
a method in the same family that naturally grows or shrinks the learning rate based on a different accumulation in the denominator, 
one that can both increase and decrease.
We show that it obeys a similar convergence rate as AdaGrad and demonstrate its non-monotone adaptation capability with experiments.
\end{abstract}

\section{Introduction}
\label{sec:intro}We consider stochastic optimization problems of the form \[
\minimize_{x\in \mathbb{R}^{d}} ~f(x):=\E_\xi[f(x,\xi)],
\]
where $\xi$ is a random variable. We assume that each  $f(x,\cdot)$ is convex in $x$ but potentially non-smooth. Problems of this class are most often addressed with stochastic gradient methods (SGD).
Aiming to accelerate the convergence of stochastic gradient methods in practice, adaptive rules for adjusting the learning rate have a long history \citep[e.g.,][]{Kesten58} and have been extensively studied more recently.
The stochastic gradient descent method, a member of the Mirror Descent (MD) family takes the general form:
\begin{equation}\label{eqn:adagrad}
x_{k+1} = x_k - A_{k+1}^{-1} g_k,
\end{equation}
where $A_{k+1}$ is a dynamically adjusted positive definite matrix.
Due to the high dimensionality of modern machine learning applications, $A_{k+1}$ is usually taken to be a diagonal matrix. AdaGrad \citep{adagrad} uses the choice:
\begin{equation}\label{eqn:adagrad-Ak}
A_{i,k+1}^{-1} = \frac{\gamma}{\sqrt{\textstyle\sum_{t=0}^k g_{i,t}^2}},
\end{equation}
where $A_{i,k+1}$ is the $i$th diagonal entry of $A_{k+1}$, $g_{i,t}$ is the $i$th coordinate of the stochastic gradient $g_t$, and $\gamma$ is a hyper-parameter.
Another popular and effective algorithm is Adam \citep{kingma2014adam}, which uses an exponential weighted average instead of a summation, in combination with momentum.
Both methods have seen widespread adoption and have been the focus of extensive followup work 
\citep[e.g.,][]{TielemanHinton2012lecture,reddi2019convergence,li_orabona_2019}.

A fundamental limitation of AdaGrad is \emph{monotone  adaptation}.
In other words, since the sum in~\eqref{eqn:adagrad-Ak} is increasing, it can only decrease the learning rate over time.
Therefore, the global learning rate parameter $\gamma$ must be carefully tuned. 
In addition, once the learning rate becomes very small, it cannot pick up speed again even if the landscape becomes flat later on.
\begin{figure}
\noindent\begin{minipage}{.4\textwidth}
\begin{algorithmic}
    \STATE {\bfseries Input:} $\gamma_0$, $\rho$, $r$, $x_0$
    \STATE Initialize $g_{-1} = 0$, $\alpha_0=0$ 
    \FOR{$k=0$ {\bfseries to} $n$}
        \STATE $g_{k} = \nabla f(x_{k},\xi_{k})$
        \STATE $v_{k} = \left\Vert g_{k}\right\Vert ^{2}-\rho\left\langle g_{k},g_{k-1}\right\rangle $
        \IF{$v_k \geq 0$}
            \STATE $\gamma_{k+1}=\gamma_{k}$
            \STATE $\alpha_{k+1}=\alpha_{k}+v_{k}$
        \ELSE
            \STATE $v_{k}=\max\left(v_{k},-r\alpha_{k}\right)$
            \STATE $\gamma_{k+1}=\gamma_{k}\sqrt{1-v_{k}/\alpha_{k}}$
            \STATE $\alpha_{k+1}=\alpha_{k}$
        \ENDIF
        \STATE $A_{k+1}=\sqrt{\alpha_{k+1}}/\gamma_{k+1}$
        \STATE $x_{k+1} =x_{k}-A_{k+1}^{-1}g_{k}$
    \ENDFOR
\end{algorithmic}
\captionof{algocf}{GradaGrad \\ (Simplified Scalar variant)}\label{alg:gradagrad-scalar}
\end{minipage}%
\begin{minipage}{.6\textwidth}
  \centering
  \includegraphics[width=1\columnwidth]{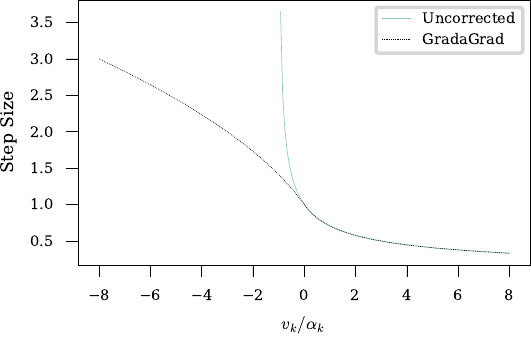}
  \captionof{figure}{The GradaGrad update (Equation~\ref{eq:negupdate})
results in a smooth scaling of the learning rate when negative values
of $v_{k}$ are encountered, compared to the naive solution from Equation~\eqref{eqn:gradagrad-Ak} which quickly explodes.}\label{fig:quadroot}
\end{minipage}
\end{figure}
In this paper, we propose a \emph{non-monotone} adaptive stochastic gradient method called Grad-GradaGrad, or GradaGrad for short.
The basic idea is to append an inner product term $\langle g_t,g_{t-1}\rangle$ to the summands in~\eqref{eqn:adagrad-Ak}.
Specifically,
\begin{equation}\label{eqn:gradagrad-Ak}
A_{i,k+1} \propto \sqrt{\textstyle\sum_{t=0}^k \left(g_{i,t}^2 - \rho g_{i,t}g_{i,t-1}\right)},
\end{equation}
where $\rho>0$ is a constant that can be tuned. 
The expectation of $-g_{i,t}g_{i,t-1}$ is the hyper-gradient of~$f$ with respect to the learning rate for the $i$th coordinate (see Section~\ref{sec:related-work}).
Intuitively, if consecutive stochastic gradients are highly positively correlated, then $g_{i,t}^2-\rho g_{i,t}g_{i,t-1}$ can be negative and thus the learning rate will be increased.
Similarly, if consecutive gradients are negatively correlated, then the learning rate will be decreased. 
In the ``goldilocks'' zone, the hyper-gradient term wobbles around 0 and the learning rate changes in a similar manner to AdaGrad. Compared to AdaGrad, our method is able to rapidly increase the learning rate when it is too small (Figure \ref{fig:abs-figure}).

However, applying the idea above directly runs into several problems. 
The most apparent problem is that the sum under the square root may become a negative number. A more subtle problem is how to ensure that the sequences of $A_{i,k+1}$ increase, in a non-monotone manner, at an appropriate rate to guarantee convergence. 
After laying out the basic ideas in Section~\ref{sec:augmenting},
we show in Section~\ref{sec:gradagrad} how the GradaGrad method addresses these problems with reparametrization and other techniques.
As a preview, Algorithm~\ref{alg:gradagrad-scalar} shows a simplified version of GradaGrad using a scalar learning rate.

In the rest of Section~\ref{sec:gradagrad}, 
we show that GradaGrad obeys a similar convergence bound as the AdaGrad method. 
In Section~\ref{sec:experiments}, preliminary experiments demonstrate that our method matches the practical performance of several benchmark methods and has the advantage of non-monotone  adaptation.

\paragraph{Notation and Assumptions}
Throughout this paper, we let $x_{*}$ denote any minimizer of $f$ and $x_{0}$ be the initial point. We use the convention that $g_{-1}$ is the zero vector.
The notation $B\succ 0$ means that $B$ is a symmetric and positive definite matrix.
For $B\succ 0$, we define $\left\Vert x\right\Vert _{B}=\sqrt{x^{T}Bx}$
and $\left\langle x,y\right\rangle _{B}=x^{T}By$. 
Some of our results will use the Lipschitz smoothness of $f$, in which case we denote the smoothness constant as $L$.
We define the filtration
$\Phi=\left(\mathcal{F}_{1},\mathcal{F}_{2},\dots\right)$, where
$\mathcal{F}_{k}$ is the $\sigma$-algebra generated by a sequence
of random variables $\{\xi_0,\ldots,\xi_{k-1}\}$. 
The notation
$\E_{\xi_k}[\,\cdot\,]$ and $\E[\,\cdot\,|\,\mathcal{F}_k]$
denote the expectation conditioned on $\mathcal{F}_k$.

\begin{figure*}[t]
\centering\subfloat[AdaGrad]{\includegraphics[width=0.98\textwidth]{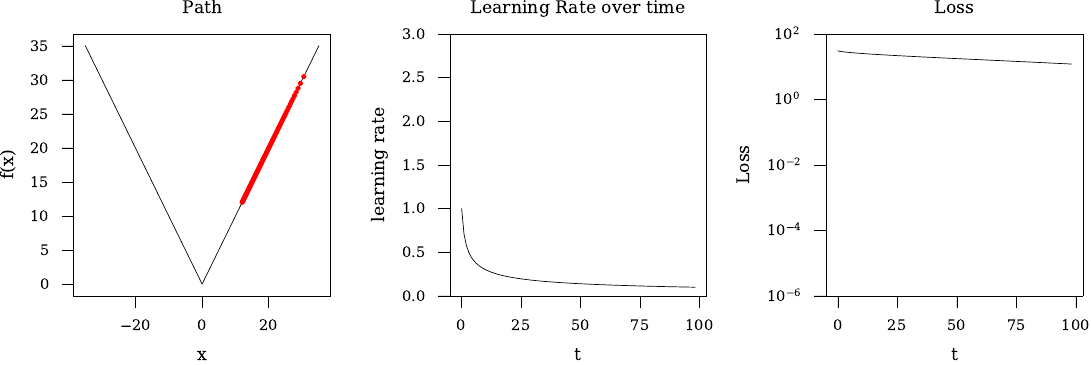}
}
\par
\subfloat[GradaGrad]{\includegraphics[width=0.98\textwidth]{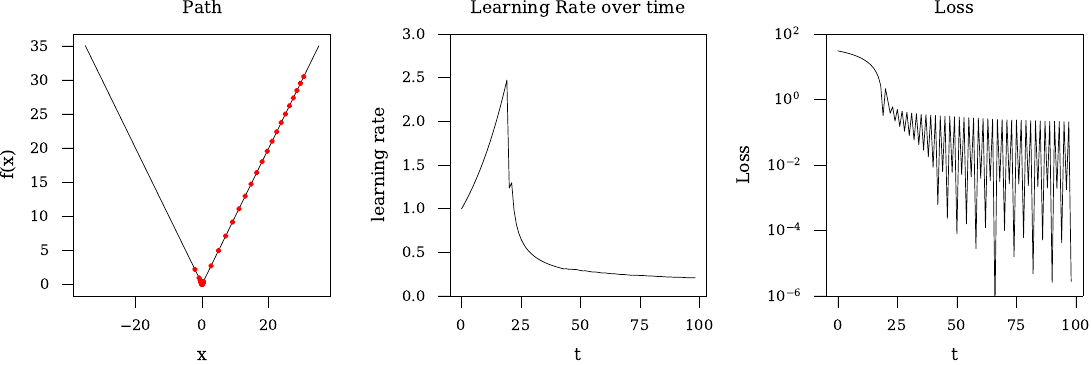}
}
\caption{\label{fig:abs-figure}
(a) Monotone learning rates cause AdaGrad to converge slowly when starting with a small initial learning rate;
(b)~GradaGrad automatically adapts to a poor initial learning rate in a non-monotone manner, leading to faster convergence.
}\vspace{-0.5em}
\end{figure*}

\section{Related Work}
\label{sec:related-work}Adaptive stochastic gradient methods have seen heavy investigation on a number of fronts, including non-monotone methods that can both increase and decrease the learning rate.
\citet{NoMorePesky2013} proposed a method that combines the estimation of gradient variance and local curvature.
\citet{CoinBetting2017NIPS} proposed a method to adjust the learning rate based on a ``coin-betting'' strategy. Methods that maintain multiple learning rates, such as Metagrad \citep{metagrad} are also potentially non-monotonic.
Another line of work that can both increase and decrease the learning rate are stochastic line search methods
\citep{Vaswani2019linesearch,PaquetteScheinberg2020,ZhangLangLiuX2020}. None of the existing strategies can be considered ``drop-in'' replacements for AdaGrad in the same way our approach is, as they require either additional gradient evaluations, additional memory overhead or knowledge of problem-dependent constants that make them harder to use in practice than the AdaGrad family of methods.

Methods in the Polyak class \citep{stochastic_polyak} compute learning rates using the function
value sub-optimality $f(x_{k})-f(x_{*})$, or in the stochastic case $f(x_k,\xi_k)-f^*_{\xi_k}$ where $f^*_{\xi_k}$ is the minimial value of $f(\cdot,\xi_k)$. These methods can exhibit
impressive performance but rely on knowledge or estimates of the minimal
function value $f_{*}$ or $f^*_{\xi_k}$, which is not available for most problems. 

The idea of adjusting the learning rate based on inner products of consecutive stochastic gradients goes back to \citet{Kesten58}, and it was extended to the multi-dimensional case by \citet{DelyonJuditsky93}. \citet{MirzoakhmedovUryasev83} showed that such  adaptation schemes can be interpreted as stochastic hyper-gradient methods.
Specifically, 
suppose $d_k$ is a search direction depending on the random variable $\xi_k$ (e.g., $g_k$ itself or combined with momentum).
Let's define a merit function of the learning rate at each iteration
\[
\phi_k(\eta) = \E\left[f(x_{k+1})\,|\,\mathcal{F}_k\right]
= \E\left[f(x_k-\eta d_k)\,|\,\mathcal{F}_k\right].
\]
The hypergradient of $\phi_k$ with respect to~$\eta$ can be derived as
\begin{align}
\frac{\partial\phi_k}{\partial \eta}
&=\E\left[\frac{\partial}{\partial \eta}f(x_k-\eta d_k)\,|\,\mathcal{F}_k\right] \nonumber 
 =\E\left[\langle \nabla f(x_k-\eta d_k), -d_k \rangle\,|\,\mathcal{F}_k\right] \nonumber\\
&=\E\left[-\langle g_{k+1}, d_k\rangle\,|\,\mathcal{F}_k\right], \label{eqn:hyper-grad}
\end{align}
where we assume necessary technical conditions for changing the order of differentiation and expectation.
Therefore, $-\langle g_{k+1}, d_k\rangle$ can be viewed as a stochastic hypergradient. 
If $d_k=g_k$, then it becomes $-\langle g_{k+1}, g_k\rangle$.
If in addition the learning rate is different for each coordinate, say $\eta_i$ for the $i$th coordinate, then it can be shown that
\[
\frac{\partial\phi_k}{\partial \eta_i}
=\E\left[-\langle g_{i,k+1}, g_{i,k}\rangle\,|\,\mathcal{F}_k\right].
\]
These are exactly the terms, after multiplying by a constant $\rho$, that are added to the summands under the square root in~\eqref{eqn:gradagrad-Ak}, although the time-step is shifted backwards by one.
In the Adam-style variant of GradaGrad we develop, $d_k$ is the convex combination of $g_k$ and a momentum term, and the additional terms in calculating $A_{i,k}$ becomes $-\rho g_{i,t} d_{i,t-1}$. 

Adaptive stochastic gradient methods based on hyper-gradient descent have been studied in the machine learning literature before, including
\citet{Jacobs88DBD}, \citet{Sutton92IDBD},
\citet{Schraudolph1999}, \citet{MahmoodSutton12TuningFree}, and more recently by \citet{Baydin18hypergradient}.
The major difference between our work and these previous work is that, although our adaptation scheme carries the intuition of hyper-gradient, our method is not really a hyper-gradient method and our analysis does not rely on the hyper-gradient interpretation. 

The ``without descent'' adaptive method of \citet{malitsky2019adaptive} is the most direct inspiration for our work. In a footnote, they suggest the following rule for learning rate $\eta_k$, when the denominator is positive:
\[
\eta_{k}^{2}\leq\frac{\left\Vert x_{k}-x_{k-1}\right\Vert ^{2}}{3\left\Vert g_k\right\Vert ^{2}-4\left\langle g_k,g_{k-1}\right\rangle }.
\]
Our method arose from attempts to adapt this learning rate to the stochastic case.

\section{Augmenting AdaGrad}
\label{sec:augmenting}For simplicity we consider the Euclidean step setting, with an unconstrained
domain. In this situation, as shown by \citet{adagrad}, the AdaGrad update
\begin{equation}\label{eq:update-1}
\begin{aligned}
g_{k} & =\nabla f(x_{k},\xi_{k}), \\
x_{k+1} & =x_{k}-A_{k+1}^{-1}g_{k},
\end{aligned}
\end{equation}
obeys the inequality:
\begin{align}
2\mathbb{E}\sum_{k=0}^{n}\left(f(x_{k})-f_{*}\right) \leq\left\Vert x_{0}-x_{*}\right\Vert _{A_{1}}^{2}
 +\mathbb{E}\sum_{k=1}^{n}\left\Vert x_{k}-x_{*}\right\Vert _{\left(A_{k+1}-A_{k}\right)}^{2}
 +\mathbb{E}\sum_{k=0}^{n}\left\Vert g_{k}\right\Vert _{A_{k+1}^{-1}}^{2}.\label{eq:adagrad-key-bound}
\end{align}
The right-hand side of this inequality consists of some initial conditions
(first term), an iterate distance error term (second term) and a gradient norm
term (third term). The key insight of the AdaGrad method is that the
gradient error term's growth can be restricted to a $\mathcal{O}(\sqrt{n})$ rate
over time if the matrices $A_k$ are carefully chosen. For simplicity of notation, consider
the one-dimensional case for the remainder of this section. 
The gradient error term grows each step by:
$
g_{k}^{2}/A_{k+1},
$
and the overall growth is controlled by using an inductive bound of
the form:
$
q A_{k}+g_{k}^{2}/A_{k+1} \leq q A_{k+1}.
$
AdaGrad uses the choice $A_{k+1}^{2}=\sum_{t=0}^{k}g_{t}^{2}$ and $q=2$ to
keep the overall growth rate $\mathcal{O}(\sqrt{n})$. In our approach, we add
additional terms which combine with the gradient error term.
Inspired by hyper-gradient methods, we use the following Lemma.
\begin{lem}
\label{lem:hyper-lemma}For any $A_{k}\succ0$ and $\rho>0$, the sequences generated by \eqref{eq:update-1} obey:
\[
0\leq-\mathbb{E}\left[\rho\left\langle g_{k},g_{k-1}\right\rangle _{A_{k}^{-1}}|\mathcal{F}_{k}\right]+\rho\left[f(x_{k-1})-f(x_{k})\right].
\]
\end{lem}
The first term is a scaling of the\textit{ hyper-gradient}; see Equation~\eqref{eqn:hyper-grad}.
Usually this quantity is used to adapt the learning rate by applying stochastic gradient descent on the learning rate, treating the hyper-parameter as a parameter (hence the name). Our approach
here is different, we introduce the term into our convergence rate
bound directly. The second term is a function value difference between
the current and previous points. This term is particularly nice as
it may be telescoped, and so it introduces an error term that doesn't
grow over time. The degree of adaptivity is controlled by $\rho$, which is a tunable parameter of our method. The theory we develop in Section~\ref{sec:rho} suggests that $\rho=2$ is the best default choice, although larger values can lead to greater adaptivity.

Using Lemma \ref{lem:hyper-lemma} and the AdaGrad result~\eqref{eq:adagrad-key-bound}, we may derive the following bound.
\begin{thm} \label{thm:telescoped-bound}
For $k\geq1$, Algorithm~\ref{alg:gradagrad-scalar} satisfies:
\begin{flalign*}
2\mathbb{E}\sum_{k=0}^{n}\left[f(x_{k})-f_{*}\right] 
~\leq~ & \left\Vert x_{0}-x_{*}\right\Vert _{A_{1}}^{2}+\rho\left[f(x_{0})-f_{*}\right]
 ~+~ \mathbb{E} \sum_{k=1}^{n}\left\Vert x_{k}-x_{*}\right\Vert _{\left(A_{k+1}-A_{k}\right)}^{2}\\
& +\mathbb{E} \sum_{k=0}^{n}\left[\left\Vert g_{k}\right\Vert _{A_{k+1}^{-1}}^{2}-\rho\left\langle g_{k-1},g_{k}\right\rangle _{A_{k}^{-1}}\right].
\end{flalign*}
\end{thm}
Compared to the bound for AdaGrad in Equation~\eqref{eq:adagrad-key-bound}, the gradient norm error term now grows each step by a term:
\[
\left\Vert g_{k}\right\Vert _{A_{k+1}^{-1}}^{2}-\rho\left\langle g_{k},g_{k-1}\right\rangle _{A_{k}^{-1}}\quad\text{v.s.}\quad\left\Vert g_{k}\right\Vert _{A_{k+1}^{-1}}^{2}.
\]
This error term is potentially much smaller, or even negative, when
consecutive gradients are highly positively correlated. This corresponds to the
situation in hyper-gradient methods where the learning rate should be
increased. This error term is larger in the opposite situation, where
consecutive gradients are negatively correlated, and the learning rate
should be decreased. In the ``goldilocks'' zone, the hyper-gradient
term will wobble around 0, and should not contribute significantly
to the accumulated error over time.

Ideally, we would like to apply a similar inductive bound to this
error term as was applied for AdaGrad:
\[
q A_{k}+\frac{\left\Vert g_{k} \right\Vert^{2}}{A_{k+1}}-\rho\frac{\left\langle g_{k}, g_{k-1} \right\rangle}{A_{k}}\leq q A_{k+1},
\]
with $A$ given by \eqref{eqn:gradagrad-Ak}. This schema is the motivation for development of our GradaGrad method.
Applied directly we run into a number of problems that must be solved
to arrive at a practical method:
\begin{enumerate}[i)]
\item The above update can potentially yield the square root of a negative
number. The approach we develop in the following sections, shown in Figure~\ref{fig:quadroot}, steadily grows the learning rate, avoiding the divergence towards infinity encountered in the naive implementation in Equation~\eqref{eqn:gradagrad-Ak}.
\item When $A_{k}$ is allowed to both expand and shrink, generally, the iterate-distance
error term, i.e., second line in~\eqref{eq:adagrad-key-bound}, may grow at a $\mathcal{O}(n)$ rate, rather than the $\mathcal{O}(\sqrt{n})$
rate required for a non-trivial convergence rate bound. It's this term that prevents the use of arbitrary step-size sequences. Our approach maintains the $\mathcal{O}(\sqrt{n})$ rate by construction.
\item The hyper-gradient term is divided by $A_{k-1}$ not $A_{k}$, which significantly breaks the bounding approach used by AdaGrad. The bound
may be violated both when the error term is positive and negative.
\end{enumerate}

\section{The GradaGrad Method}
\label{sec:gradagrad}

In this section, we show how to address the issues listed above
through a reparametrization technique, present the full version of GradaGrad with momentum, and establish its convergence properties.

\subsection{Controlling error through reparameterization}
\label{sec:reparametrization}
The key to controlling the error terms that arise in the GradaGrad
method is the use of reparameterization of the learning rate. Like AdaGrad,
our learning rate will take the form:
\[
A_{k+1}^{-1}=\frac{\gamma_{k+1}}{\sqrt{\alpha_{k+1}}},
\]
where $\alpha_{k+1}$ is updated from $\alpha_k$ each step by adding some additional term $v_{k}$. 
In GradaGrad, we allow the numerator $\gamma_{k+1}$ to change over time, compared to the fixed numerator used in AdaGrad. The purpose of this additional flexibility is to allow us to
reparameterize the learning rate before applying the AdaGrad like additive
update to $\alpha_{k}$. We still consider coordinate-wise updates of the form:
\begin{equation}
\alpha_{k+1}=\alpha_{k}+v_{k}, \label{eq:standard}
\end{equation}
where $v_{k}=\left\Vert g_{k}\right\Vert^{2}$ in AdaGrad and 
$v_{k}=\left\Vert g_{k}\right\Vert^{2} - \rho\left\langle g_{k},g_{k-1}\right\rangle$
in GradaGrad. However, we apply these updates after a reparameterization:
\begin{equation}
\frac{\gamma_{k+1}}{\sqrt{\alpha_{k+1}^{\prime}}}=\frac{\gamma_{k}}{\sqrt{\alpha_{k}}}\label{eq:reparam},
\end{equation}
that leaves the learning rate the same but changes the effect of adding
$v_{k}$. The update then consists of $\alpha_{k+1}=\alpha_{k+1}^{\prime}+v_{k}$. We choose our reparameterization so that when $v_k$ is negative, $\alpha_{k+1}=\alpha_k$ regardless of the value of $v_k$. Solving Equation \eqref{eq:reparam} under this condition gives the update
\begin{align}
\gamma_{k+1} &  
=\gamma_{k}\sqrt{\frac{\alpha'_{k+1}}{\alpha_{k}}}
=\gamma_{k}\sqrt{\frac{\alpha_{k+1}-v_{k}}{\alpha_{k}}}
=\gamma_{k}\sqrt{\frac{\alpha_{k}-v_{k}}{\alpha_{k}}}
=\gamma_{k}\sqrt{1-\frac{v_{k}}{\alpha_{k}}}.
\label{eq:negupdate}
\end{align}
This results in the behavior observed in Figure \ref{fig:quadroot}, where the learning rates smoothly increases as $v_k$ becomes more negative. Note that this correction is only used when $v_k$ is negative, for positive $v_k$ the standard AdaGrad update in Equation~\eqref{eq:standard} is used without any reparameterization. This update is very well behaved, the learning rate matches the value from using $\alpha_{k+1} =\alpha_{k} - v_{k}$ up to a first order approximation, as the gradients match at $v_k=0$ (this can be seen in Figure \ref{fig:quadroot}). The ratio $v_k/\alpha_{k}$ is typically very small at the later stages of optimization, so its behavior differs from the naive variant only at the beginning of optimization.

\begin{figure}[t]\centering
\begin{minipage}[t]{.5\textwidth}%
\centering
\begin{algorithmic}
    \STATE {\bfseries Input:} $\gamma_0$, $\rho$, $x_0$, $\beta$, $G_{\infty}$, $D_\infty$
    \STATE Initialize $z_0 = x_0$
    \FOR{$k=0$ {\bfseries to} $n$}
        \STATE $g_{k} = \nabla f(x_{k},\xi_{k})$
        \FOR{$i=1$  {\bfseries to} $D$}
        \IF{$k=0$}
            \STATE $v_{i,k}=G_{\infty}^{2}$
        \ELSIF{$\gamma_{i,k} = D_{\infty}$}
            \STATE $v_{i,k}=g_{i,k}^{2}$
        \ELSE
            \STATE $v_{i,k}=g_{i,k}^{2}-\rho g_{i,k} m_{i,k-1}$
        \ENDIF
        \IF{$v_{i,k} \geq 0$}
            \STATE $\gamma_{i,k+1}=\gamma_{i,k}$
            \STATE $\alpha_{i,k+1}=\alpha_{i,k} + v_{i,k}$
        \ELSE
            \STATE $r_{i,k}=\left(\frac{\rho m_{i,k-1}}{g_{i,k}}\right)^{2}-1$
            \STATE $v_{i,k}=\max\left(v_{i,k},-r_{i,k}\alpha_{i,k}\right)$
            \STATE $\gamma_{i,k+1}=\min\left(\gamma_{i,k}\sqrt{1-\frac{v_{i,k}}{\alpha_{i,k}}}, \, D_{\infty}\right)$
            \STATE $\alpha_{i,k+1}=\alpha_{i,k}$
        \ENDIF
        \STATE $A_{i,k+1}=\sqrt{\alpha_{i,k+1}}/\gamma_{i,k+1}$
        \ENDFOR
        \STATE $z_{k+1}=\text{proj}_{\mathcal{X}}(z_{k}-A_{k+1}^{-1}g_{k})$
        \STATE $x_{k+1}=\beta x_{k}+\left(1-\beta\right) z_{k+1}$
        \STATE $m_{k}=A_{k+1}\left(x_{k}-x_{k+1}\right)$
    \ENDFOR
\end{algorithmic}
\captionof{algocf}{GradaGrad}\label{alg:gradagrad-diagonal}
\end{minipage}%
\vspace{-1em}
\end{figure}

\subsection{GradaGrad with momentum}
\vspace{-0.5em}\label{sec:method} The scalar version of GradaGrad is detailed in Algorithm \ref{alg:gradagrad-scalar}.
It's straight-forward to extend the method to per-coordinate adaptivity (diagonal scaling), and to include the use of momentum and projection onto a convex set $\mathcal{X}$. This full version is given in Algorithm \ref{alg:gradagrad-diagonal}. The full version includes 2 other changes needed to facilitate the analysis: firstly, the $v_0$ update uses a maximum element-wise gradient bound $G$ rather than the observed gradient $g_0$. This kind of change is also needed by other AdaGrad variants (such as AdaGrad-DA) for the theory and can usually be omitted from practical implementations of the method.

Secondly, we derive an explicit expression for the hyper-parameter $\rho$ rather than using a fixed value. This is a trade-off, as it adds some additional complexity to the method but reduces the number of hyper-parameters. The necessity of this parameter is detailed in Section \ref{sec:restricted-increase}. We establish the following $\mathcal{O}(1/\sqrt{n})$ convergence rate bound for the diagonal scaling variant, the scalar variant can be analysed using the same techniques.


\begin{thm}
\label{thm:main}Define $G_{\infty}=\max_{k}\left\Vert g_{k}\right\Vert_{\infty}$, 
$D_{\infty} \geq \sup_{x\in\mathcal{X}}\left\Vert x-x_{*}\right\Vert_{\infty}$ then the function value at the average iterate
$\bar{x}$ of GradaGrad converges at a $\mathcal{O}(1/\sqrt{n})$ rate:
\begin{align*}
\mathbb{E}\left[f(\bar{x}_{n+1})-f_{*}\right] & \leq\frac{1}{2(n+1)}\left(\frac{2\beta}{1-\beta}+\rho\right)\left[f(x_{0})-f_{*}\right] + 3\sqrt{2+\rho}\,\frac{dG_{\infty}D_{\infty}^{2}}{\gamma_{0}\sqrt{n+1}}.
\end{align*}
where $\rho>0$ is a hyper-parameter and $\beta\in[0,1)$ is the classical momentum parameter.
\end{thm}

\subsection{The restricted increase constraint}
\label{sec:restricted-increase} As mentioned in point (iii), care must be taken to control the growth of the gradient noise term:
\[
\left\Vert g_{k}\right\Vert _{A_{k+1}^{-1}}^{2}-\rho\left\langle g_{k-1},g_{k}\right\rangle _{A_{k}^{-1}}.
\]
The AdaGrad approach to bounding the accumulated error:
\[
q A_{k}+\left\Vert g_{k}\right\Vert _{A_{k+1}^{-1}}^{2}-\rho\left\langle g_{k-1},g_{k}\right\rangle _{A_{k}^{-1}}\leq q A_{k+1}.
\]
breaks down when 
$v_{k}=\left\Vert g_{k}\right\Vert^{2} - \rho\left\langle g_{k},g_{k-1}\right\rangle$
is negative,
as this inequality
can become impossible to satisfy for any choice of $A_{k}\prec A_{k-1}$,
preventing us from using increasing learning rate sequences. Instead,
we will impose the following restriction:
\begin{equation}
\left\Vert g_{k}\right\Vert _{A_{k+1}^{-1}}^{2}-\rho\left\langle g_{k-1},g_{k}\right\rangle _{A_{k}^{-1}}\leq0.\label{eq:errnegativity}
\end{equation}
This is not always satisfied, even when $\left\Vert g_{k}\right\Vert ^{2}-\rho\left\langle g_{k-1},g_{k}\right\rangle \leq0$,
as the gradient norm square term may grow very large if the learning
rate $A_{k}^{-1}$ is allowed to increase significantly between steps. By rearranging terms, we see that this bound is satisfied when the ratio of  $A_{k+1}$ to $A_{k}$ is controlled:
\[
A_{k+1}A_{k}^{-1}\succeq\frac{\left\Vert g_{k}\right\Vert ^{2}}{\rho\left\langle g_{k-1},g_{k}\right\rangle }.
\]
We satisfy this constraint by using the following update in the GradaGrad method:
\begin{lem} \label{lem:hbound} Inequality \eqref{eq:errnegativity} is satisfied if we clip $v_k$ using:
\[
v_{k}=\max\left(v_{k},-r\alpha_{k}\right) \,\,\text{where}\,\, r=\left(\frac{\rho\left\langle g_{k-1},g_{k}\right\rangle }{\left\Vert g_{k}\right\Vert ^{2}}\right)^{2}-1.
\]
\end{lem}
\subsection{The \texorpdfstring{$\rho$}{rho} hyper-parameter}
\vspace{-0.5em}\label{sec:rho}There are two values of $\rho$ for which the update has special properties
worth further discussion. Consider the use of $\rho=1$. The $\alpha_{k}$
sequence (without reparameterization) can be written as 
\begin{align*}
\alpha_{n+1} & =g_{0}^{2}+\sum_{k=1}^{n}\left[g_{k}^{2}-g_{k}g_{k-1}\right] =\frac{1}{2}g_{0}^{2}+\frac{1}{2}g_{n}^{2}+\frac{1}{2}\sum_{k=0}^{n-1}\left(g_{k+1}-g_{k}\right)^{2}.
\end{align*}
Although individual steps may increase or decrease the rate, the overall growth will be dominated by the summation term, and so GradaGrad with $\rho=1$ will not result in significant learning rate growth. The accumulation of differences of gradients instead of gradients has been explored before, under the name of optimistic online gradient descent \citep{predictable_sequences, bachlevy2019, antonakopoulos2021adaptive}.

This method is
still of interest as it can be used as a drop-in replacement for AdaGrad,
without the need for the complexities of reparameterization that larger
$\rho$ require. 
It can behave quite differently than AdaGrad, for
instance on the absolute value function example in Figure \ref{fig:abs-figure}, it will
make much more rapid progress, the learning rate initially stays constant
rather than decreasing as $g_{k+1}-g_{k}$ is zero until it overshoots
the minimum, and starts decreasing from there.

The case of $\rho=2$ is the most natural as it results in simplifications of the constants in the bounds. We would recommend $\rho=2$ as the default choice in the GradaGrad method, as it results in adaptivity without introducing excessive instability during training.

\begin{figure*}[t]
\centering\includegraphics[width=0.99\textwidth]{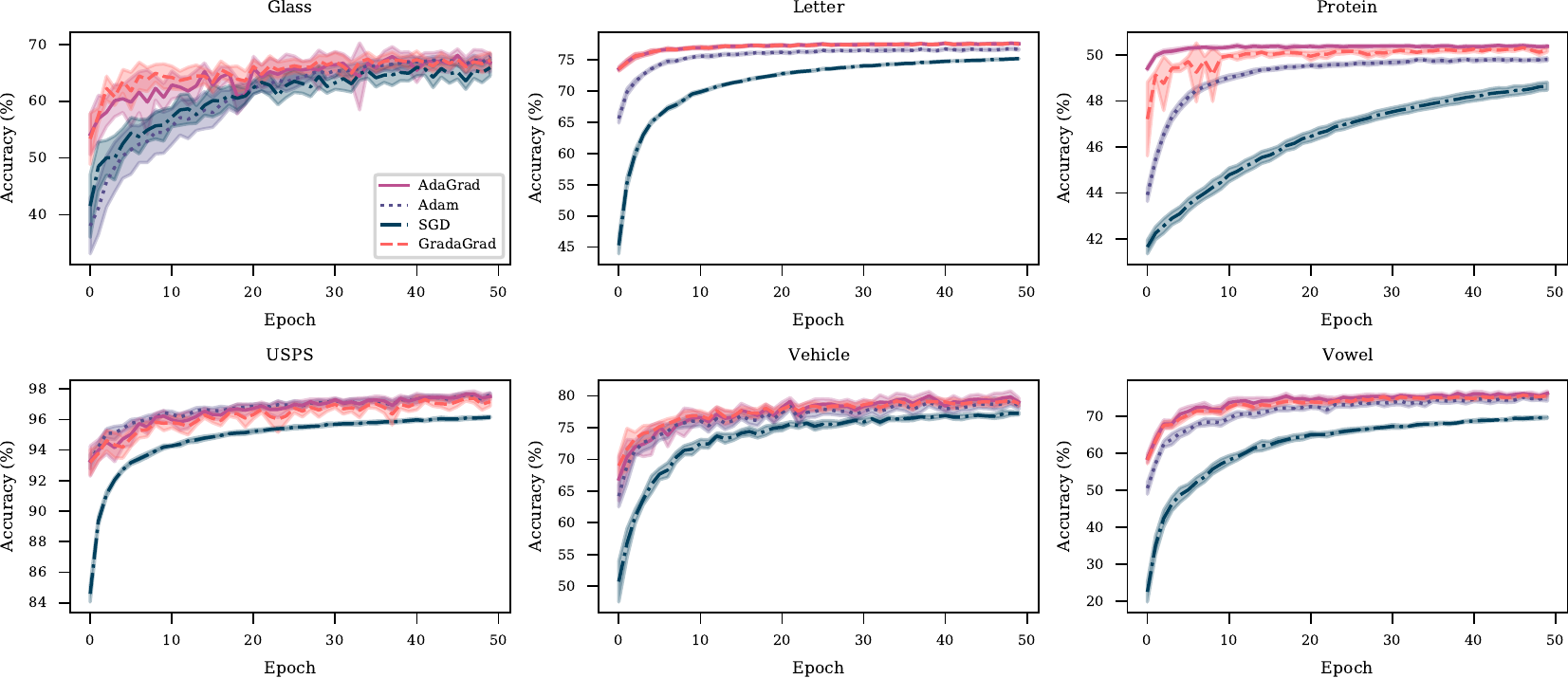}
\caption{\label{fig:convex-figure}GradaGrad matches the training accuracy of benchmark methods without tuning (with default hyper-parameters $\gamma_0=1$ and $\rho=2$).}
\end{figure*}

\begin{figure*}[t]
\centering\includegraphics[width=0.99\textwidth]{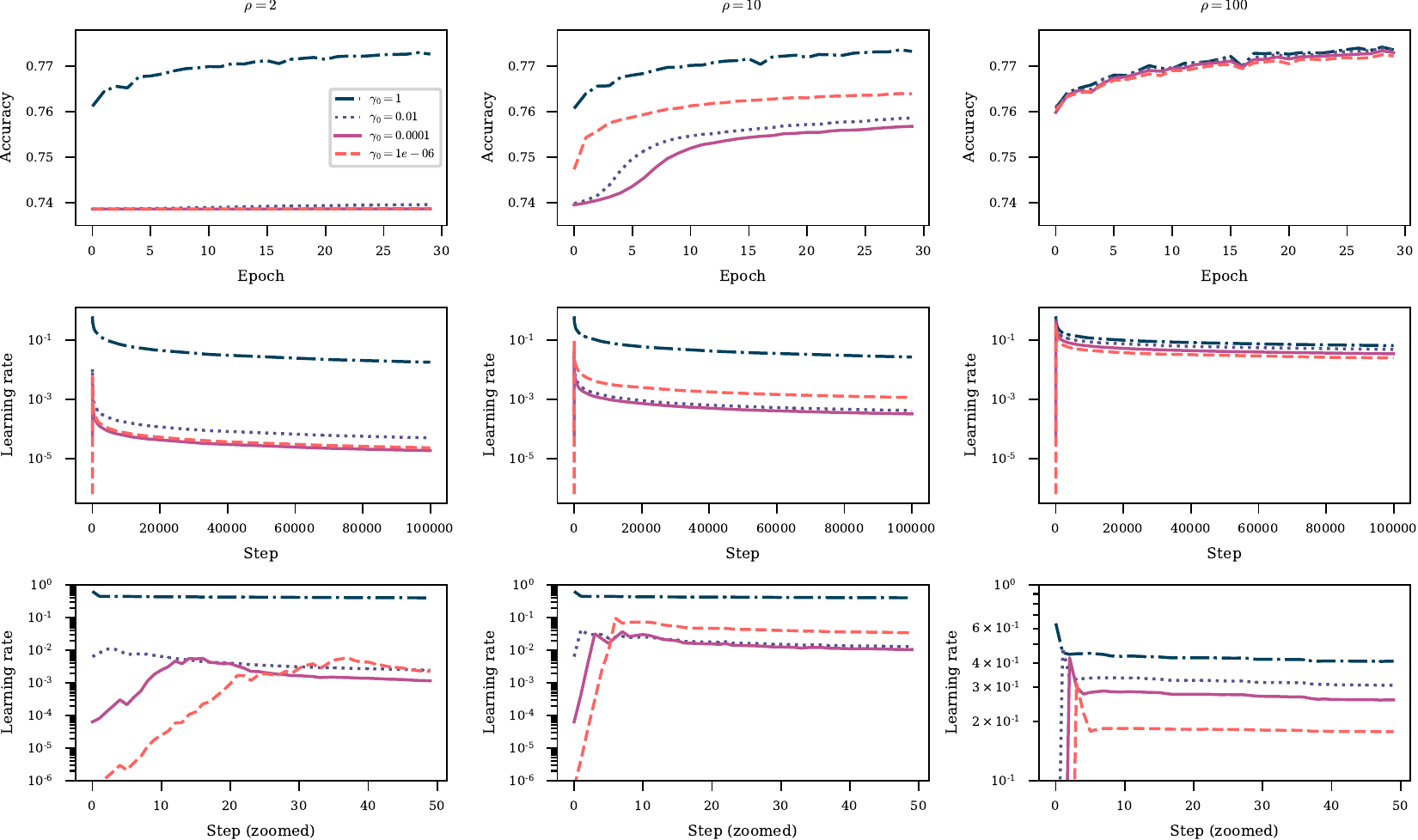}
\caption{\label{fig:mult-comparison} $\rho$ has a significant effect on the adaptability of Scalar GradaGrad on large deep learning problems such as DLRM. Our recommended default value of $\rho=2$ is highly stable. Larger values of $\rho$ give a higher degree of adaptivity but may result in less stable training.}\vspace{-1.5em}
\end{figure*}

\section{Experiments}%
\label{sec:experiments} We evaluated GradaGrad on 6 baseline problems widely used in the optimization literature \citep{chang2011libsvm, Dua:2019}): Glass, Letter, Protein, USPS, Vehicle and Vowel. We tested a binary logistic regression model, with no further data preprocessing, using the scaled versions of each dataset (when applicable) retrieved from the LIBSVM dataset repository. We compare our method against Adam, SGD and AdaGrad, where for each method the learning rate is chosen via grid search on a power-of-2 grid to give the highest final  accuracy averaged over the last 10 epochs. This tuning is extremely advantageous to these baselines methods, and provides a difficult benchmark to beat. We plot the average of 10 runs with different random seeds, with a 2 standard-error range overlaid. As shown in Figure \ref{fig:convex-figure}, GradaGrad is able to match these baseline methods using the same default hyper-parameters of $\gamma_0=1.0$, and $\rho=2$ across all problems.

To illustrate the effect of the $\rho$ hyperparameter, we performed experiments on the DLRM \citep{DLRM19} model on the  Criteo Kaggle Display Advertising Challenge Dataset, using the scalar variant of GradaGrad. This model has more than 300 million parameters and is representative of models used in industry. Batch size 128 and embedding dimension 16 were used, with no regularization.  Whereas the diagonal scaling variant of GradaGrad performs well with $\rho=2$, the scalar variant requires larger $\rho$ for large or high noise problems. Figure \ref{fig:mult-comparison} shows the results of using $\gamma_0$ equal to 1, 1e-2, 1e-4, 1e-6 with different $\rho$ values. Values of $\rho=2$ and $\rho=10$ show a significant lack of adaptivity, whereas for large $\rho$ each of these rates performs similarly. 

\section{Discussion}
\vspace{-0.3em}\label{sec:discussion}Although GradaGrad's learning rate is more adaptive than that of AdaGrad, it still inherits many of the same limitations. The ``any-time'' sequences that have a square-root growth denominator give optimal rates up to a constant factor, but they sometimes give noticeably worse rates in practice than more heavily hand-tuned learning rate schedules such as MADGRAD. \citep{defazio2021adaptivity}.

The GradaGrad method is also limited in its ability to adjust $\gamma$. The updates only increase $\gamma$, so the initial $\gamma_0$ must be no larger than the optimal $\gamma$. The GradaGrad update allows $\gamma$ to grow at an exponential rate, which means that small initial values can potentially be used. Care must be taken when initializing with $\gamma_0$ that are very small as numerical issues may come into play. Values below roughly $10^{-6}$ will likely result in loss of precision and no adaptivity when using single precision floating point.

\section*{Conclusion}
\vspace{-0.3em}The GradaGrad method has promising theoretical and practical properties that make it a viable drop-in replacement for the AdaGrad method. It exhibits greater adaptivity to the gradient sequence as it is able to exploit correlations between consecutive gradients, going beyond the simple gradient magnitudes used in the AdaGrad method. There are many open problems relating to the $v_k$ sequence we described. Can the same sequence be used in other AdaGrad family methods? We believe that $v_k$ sequences of the form we describe here could have widespread applications. Can the upper bound on $\gamma$ be removed? It does not appear to be necessary in practice.

\clearpage
\bibliography{adaptive2022}
\bibliographystyle{icml2022}

\clearpage
\appendix
\section{Convergence Theory}
We define the ancillary quantity $z^{\prime}$, so that the update
in Algorithm~\ref{alg:gradagrad-diagonal} may be written as
\begin{align}
z_{k+1}^{\prime} & =z_{k}-A_{k+1}^{-1}g_{k}, \nonumber \\
z_{k+1} & =\text{proj}_{\mathcal{X}}(z_{k+1}^{\prime}), \nonumber \\
x_{k+1} & =\beta x_{k}+\left(1-\beta\right)z_{k+1}. \label{eq:update}
\end{align}
Then define $m_{k}$ via
\[
x_{k+1}=x_{k}-A_{k+1}^{-1}m_{k}.
\]
Note that $m_{k}=A_{k}\left(x_{k}-x_{k+1}\right)$. Note also that:
\[
z_{k}=\frac{1}{1-\beta}x_{k}-\frac{\beta}{1-\beta}x_{k-1}.
\]

\begin{lem}
\label{lem:hypergrad-1-2}For $k\geq1$:
\[
\mathbb{E}_{\xi_{k}}\left\langle g_{k},m_{k-1}\right\rangle _{A_{k}^{-1}}\leq\left[f(x_{k-1})-f(x_{k})\right].
\]
\end{lem}

\begin{proof}
Using $x_{k}=x_{k-1}-A_{k}^{-1}m_{k-1}$ , 
\begin{align*}
\mathbb{E}_{\xi_{k}}\left\langle g_{k},m_{k-1}\right\rangle_{A_{k}^{-1}}  & =\mathbb{E}_{\xi_{k}}\left\langle g_{k},x_{k-1}-x_{k}\right\rangle \\
 & =\left\langle \nabla f(x_{k}),x_{k-1}-x_{k}\right\rangle \\
 & \leq f(x_{k-1})-f(x_{k}).
\end{align*}
\end{proof}
\begin{lem}
\label{lem:onestep}Consider $k\geq1$ and $\rho\geq1$. Then the
iterate sequence generated by Equation \eqref{eq:update} obeys the
following bound:
\begin{align*}
2\left[f(x_{k})-f_{*}\right] & \leq\mathbb{E}_{\xi_{k}}\left\Vert z_{k}-x_{*}\right\Vert _{A_{k+1}}^{2}-\mathbb{E}_{\xi_{k}}\left\Vert z_{k+1}-x_{*}\right\Vert _{A_{k+1}}^{2} \\
& -\left(\frac{2\beta}{1-\beta}+\rho\right)\left[f(x_{k})-f_{*}\right]+\left(\frac{2\beta}{1-\beta}+\rho\right)\left[f(x_{k-1})-f_{*}\right]\\
 & +\mathbb{E}_{\xi_{k}}\left[\left\Vert g_{k}\right\Vert _{A_{k+1}^{-1}}^{2}-\rho\left\langle g_{k},m_{k-1}\right\rangle _{A_{k}^{-1}}\right].
\end{align*}
\end{lem}

\begin{proof}
We start by expanding $x_{k+1}$ around $x_{*}$.
\begin{align*}
& \mathbb{E}_{\xi_{k}}\left\Vert z_{k+1}-x_{*}\right\Vert _{A_{k+1}}^{2}\\
& \leq\mathbb{E}_{\xi_{k}}\left\Vert z_{k+1}^{\prime}-x_{*}\right\Vert _{A_{k+1}}^{2}\\
 & =\mathbb{E}_{\xi_{k}}\left\Vert z_{k}-x_{*}+z_{k+1}^{\prime}-z_{k}\right\Vert _{A_{k+1}}^{2}\\
 & =\mathbb{E}_{\xi_{k}}\left\Vert z_{k}-x_{*}\right\Vert _{A_{k+1}}^{2}+2\mathbb{E}_{\xi_{k}}\left\langle A_{k+1}\left(z_{k+1}^{\prime}-z_{k}\right),z_{k}-x_{*}\right\rangle +\mathbb{E}_{\xi_{k}}\left\Vert z_{k+1}^{\prime}-z_{k}\right\Vert _{A_{k+1}}^{2}\\
 & =\mathbb{E}_{\xi_{k}}\left\Vert z_{k}-x_{*}\right\Vert _{A_{k+1}}^{2}-2\mathbb{E}_{\xi_{k}}\left\langle g_{k},z_{k}-x_{*}\right\rangle +\mathbb{E}\left\Vert g_{k}\right\Vert _{A_{k+1}^{-1}}^{2}\\
 & =\mathbb{E}_{\xi_{k}}\left\Vert z_{k}-x_{*}\right\Vert _{A_{k+1}}^{2}-2\left\langle \nabla f(x_{k}),\frac{1}{1-\beta}x_{k}-\frac{\beta}{1-\beta}x_{k-1}-x_{*}\right\rangle +\mathbb{E}_{\xi_{k}}\left\Vert g_{k}\right\Vert _{A_{k+1}^{-1}}^{2}\\
 & =\mathbb{E}_{\xi_{k}}\left\Vert z_{k}-x_{*}\right\Vert _{A_{k+1}}^{2}-2\left\langle \nabla f(x_{k}),\frac{\beta}{1-\beta}x_{k}-\frac{\beta}{1-\beta}x_{k-1}+x_{k}-x_{*}\right\rangle +\mathbb{E}_{\xi_{k}}\left\Vert g_{k}\right\Vert _{A_{k+1}^{-1}}^{2}\\
 & =\mathbb{E}_{\xi_{k}}\left\Vert z_{k}-x_{*}\right\Vert _{A_{k+1}}^{2}-2\left\langle \nabla f(x_{k}),x_{k}-x_{*}\right\rangle -\frac{2\beta}{1-\beta}\left\langle \nabla f(x_{k}),x_{k}-x_{k-1}\right\rangle +\mathbb{E}_{\xi_{k}}\left\Vert g_{k}\right\Vert _{A_{k+1}^{-1}}^{2}\\
 & \leq\mathbb{E}_{\xi_{k}}\left\Vert z_{k}-x_{*}\right\Vert _{A_{k+1}}^{2}-2\left[f(x_{k})-f_{*}\right]-\frac{2\beta}{1-\beta}\left[f(x_{k})-f(x_{k-1})\right]+\mathbb{E}_{\xi_{k}}\left\Vert g_{k}\right\Vert _{A_{k+1}^{-1}}^{2}
\end{align*}

We combine with Lemma \ref{lem:hypergrad-1-2} multiplied by $\rho$:
\begin{align*}
\mathbb{E}_{\xi_{k}}\left\Vert z_{k+1}-x_{*}\right\Vert _{A_{k+1}}^{2} & \leq\mathbb{E}_{\xi_{k}}\left\Vert z_{k}-x_{*}\right\Vert _{A_{k+1}}^{2}-2\left[f(x_{k})-f_{*}\right]+\left(\frac{2\beta}{1-\beta}+\rho\right)\left[f(x_{k-1})-f(x_{k})\right]\\
 & +\mathbb{E}_{\xi_{k}}\left[\left\Vert g_{k}\right\Vert _{A_{k+1}^{-1}}^{2}-\rho\left\langle g_{k},m_{k-1}\right\rangle _{A_{k}^{-1}}\right].
\end{align*}
\end{proof}

\subsection{Proof of Theorem \ref{thm:telescoped-bound}}
\label{sec:thm-tele-proof} We prove a more general form of Theorem \ref{thm:telescoped-bound}, covering the use of momentum and projection. The variant listed in the body of the paper is the special case where $\mathcal{X}=\mathbb{R}^d$ and $\beta=0$. Consider the case where $k=0$. Then:
\begin{align*}
\mathbb{E}_{\xi_{0}}\left\Vert z_{1}-x_{*}\right\Vert _{A_{1}}^{2} & \leq\mathbb{E}_{\xi_{0}}\left\Vert z_{1}^{\prime}-x_{*}\right\Vert _{A_{1}}^{2}\\
 & =\mathbb{E}_{\xi_{0}}\left\Vert z_{0}-x_{*}+z_{1}^{\prime}-z_{0}\right\Vert _{A_{1}}^{2}\\
 & =\mathbb{E}_{\xi_{0}}\left\Vert z_{0}-x_{*}\right\Vert _{A_{1}}^{2}+2\left\langle A_{1}\left(z_{1}^{\prime}-z_{0}\right),z_{0}-x_{*}\right\rangle +\mathbb{E}_{\xi_{0}}\left\Vert z_{1}^{\prime}-z_{0}\right\Vert _{A_{1}}^{2}\\
 & =\mathbb{E}_{\xi_{0}}\left\Vert z_{0}-x_{*}\right\Vert _{A_{1}}^{2}-2\mathbb{E}\left\langle g_{0},z_{0}-x_{*}\right\rangle +\mathbb{E}_{\xi_{0}}\left\Vert g_{0}\right\Vert _{A_{1}^{-1}}^{2}\\
 & \leq\mathbb{E}_{\xi_{0}}\left\Vert z_{0}-x_{*}\right\Vert _{A_{1}}^{2}-2\left[f(x_{0})-f_{*}\right]+\mathbb{E}_{\xi_{0}}\left\Vert g_{0}\right\Vert _{A_{1}^{-1}}^{2}.
\end{align*}

Using the law of total expectation, we telescope Lemma \ref{lem:onestep}
together with the base case from $k=0$ to $n$:
\begin{align*}
2\mathbb{E}\sum_{k=0}^{n}\left[f(x_{k})-f_{*}\right] & \leq\mathbb{E}\left\Vert z_{0}-x_{*}\right\Vert _{A_{1}}^{2}+\left(\frac{2\beta}{1-\beta}+\rho\right)\left[f(x_{0})-f_{*}\right]\\
 & +\mathbb{E}\sum_{k=1}^{n}\left\Vert z_{k}-x_{*}\right\Vert _{\left(A_{k+1}-A_{k}\right)}^{2}\\
 & +\mathbb{E}\sum_{k=0}^{n}\left[\left\Vert g_{k}\right\Vert _{A_{k+1}^{-1}}^{2}-\rho\left\langle g_{k},m_{k-1}\right\rangle _{A_{k}^{-1}}\right].
\end{align*}

\begin{lem}
\label{lem:iterate-bound-lemma}
Define $D_{\infty}=\sup_{x\in\mathcal{X}}\left\Vert x-x_{*}\right\Vert$. The GradaGrad update satisfies:
\[
\left\Vert z_{0}-x_{*}\right\Vert _{A_{1}}^{2}+\mathbb{E}\sum_{k=1}^{n}\left[\left\Vert z_{k}-x_{*}\right\Vert _{\left(A_{k+1}-A_{k}\right)}^{2}\right]\leq 
\frac{D_{\infty}^{2}}{\gamma_{0}}\mathbb{E}\sum_{i=1}^{d}\sqrt{\alpha_{i,n+1}}.
\]
\end{lem}
\begin{proof}
We consider the 1D case for simplicity as the bound is separable in
the dimension of the problem. This allows us to simplify the left
hand side to
\[
\left\Vert z_{0}-x_{*}\right\Vert _{A_{0}}^{2}+\mathbb{E}\sum_{k=1}^{n}\left[\left(\frac{\sqrt{\alpha_{k+1}}}{\gamma_{k+1}}-\frac{\sqrt{\alpha_{k}}}{\gamma_{k}}\right)\left\Vert z_{k}-x_{*}\right\Vert ^{2}\right].
\]
We will prove this by induction. Consider the base case, $v_{0}$ is always
positive so:
\[
\left\Vert z_{0}-x_{*}\right\Vert _{A_{0}}^{2}\leq\frac{D_{\infty}^{2}}{\gamma_{0}}\sqrt{v_{0}}\leq \frac{D_{\infty}^{2}}{\gamma_{0}}\sqrt{\alpha_{1}}.
\]
Now consider $k>0$. We have two cases corresponding to the sign of
$v_{k}$. If $v_{k}$ is negative, then $\gamma_{k+1}$ and $\alpha_{k+1}$
are rescaled, and $A_{k}$ decreases from $A_{k-1}$ so $\frac{\sqrt{\alpha_{k+1}}}{\gamma_{k+1}}-\frac{\sqrt{\alpha_{k}}}{\gamma_{k}}$
is negative, and so:
\[
\frac{D_{\infty}^{2}}{\gamma_{0}}\mathbb{E}\sqrt{\alpha_{k}}+\mathbb{E}\left(\frac{\sqrt{\alpha_{k+1}}}{\gamma_{k+1}}-\frac{\sqrt{\alpha_{k}}}{\gamma_{k}}\right)\left\Vert z_{k}-x_{*}\right\Vert ^{2}\leq\frac{D_{\infty}^{2}}{\gamma_{0}}\mathbb{E}\sqrt{\alpha_{k+1}}.
\]
Now consider the remaining case, where $v_{k}$ is positive, in which case $\gamma$ increases and $\alpha$ stays unchanged:
\begin{align*}
\frac{D_{\infty}^{2}}{\gamma_{0}}\mathbb{E}\sqrt{\alpha_{k}}+\mathbb{E}\left(\frac{\sqrt{\alpha_{k+1}}}{\gamma_{k+1}}-\frac{\sqrt{\alpha_{k}}}{\gamma_{k}}\right)\left\Vert z_{k}-x_{*}\right\Vert ^{2} & \leq\frac{D_{\infty}^{2}}{\gamma_{0}}\mathbb{E}\sqrt{\alpha_{k}}\\
 & =\frac{D_{\infty}^{2}}{\gamma_{0}}\mathbb{E}\sqrt{\alpha_{k+1}}
\end{align*}
\end{proof}

\subsection{Gradient error bounding}
\begin{lem}
\label{lem:r} When $v_{k}\leq0$:
\[
\left\Vert g_{k}\right\Vert _{A_{k+1}^{-1}}^{2}-\rho\left\langle m_{k-1},g_{k}\right\rangle _{A_{k}^{-1}}\leq0,
\]
when
\[
v_{k} \geq -r\alpha_{k},
\]
\[
r=\frac{1}{h_{k}^{2}}-1,
\]
\[
h_{k}=\frac{g_{k}^{2}}{\rho m_{k-1}g_{k}}.
\]
\end{lem}
We again consider the 1D case as the inequality is separable in the
dimension. So the required inequality is:
\begin{equation}
\frac{g_{k}^{2}}{A_{k+1}}-\rho\frac{m_{k-1}g_{k}}{A_{k}}\leq0.\label{eq:1d-akbound-1}
\end{equation}
Rearranging:
\[
g_{k}^{2}\leq\rho \frac{A_{k+1}}{A_{k}} m_{k-1}g_{k},
\]
\[
\frac{A_{k+1}}{A_{k}}\geq\frac{g_{k}^{2}}{\rho m_{k-1}g_{k}},
\]
Define: 
\[
h_{k}=\frac{g_{k}^{2}}{\rho m_{k-1}g_{k}}\geq0.
\]
We want to find a value of $r$ such that:
\[
\frac{A_{k+1}}{A_{k}}\geq h_{k}.
\]
Using $v_{k}=-r_{k}\alpha_{k}$ we have:
\[
\frac{\frac{\sqrt{\alpha_{k+1}}}{\gamma_{k+1}}}{\frac{\sqrt{\alpha_{k}}}{\gamma_{k}}}\geq h_{k},
\]
\[
\frac{\frac{\sqrt{\alpha_{k}}}{\gamma_{k}\sqrt{\frac{\alpha_{k}-v_{k}}{\alpha_{k}}}}}{\frac{\sqrt{\alpha_{k}}}{\gamma_{k}}}\geq h_{k},
\]

\[
\frac{1}{\sqrt{\frac{\alpha_{k}-v_{k}}{\alpha_{k}}}}\geq h_{k},
\]
\[
\frac{1}{h_{k}}\geq\sqrt{\left(1+r\right)},
\]
\[
r\leq\frac{1}{h_{k}^{2}}-1.
\]
Therefore any $r$ value smaller than $1/h_k^2 - 1$ is sufficient.

\begin{lem}
\label{lem:grecursive_bound}Consider the 1D case, with $k\geq1$. Let $q\geq 2\sqrt{2+\rho }\gamma_{k+1}^{2}$ if $v_{k}\geq0$:
\[
q A_{k}+\frac{g_{k}^{2}}{A_{k+1}}-\frac{\rho g_{k}m_{k-1}}{A_{k}}\leq q A_{k+1}.
\]
\end{lem}
\begin{proof}
We start by applying a case argument. Consider when $-g_{k}m_{k-1}\geq0$. From the fact that $\alpha_0=G_{\infty}^2$ we have $\frac{1}{A_{k}}\leq\frac{\sqrt{2+\rho }}{A_{k+1}}$, and therefore:
\begin{align*}
q A_{k}+\frac{g_{k}^{2}}{A_{k+1}}-\frac{\rho g_{k}m_{k-1}}{A_{k}} & \leq q A_{k}+\frac{g_{k}^{2}}{A_{k+1}}-\sqrt{2+\rho}\frac{\rho g_{k}m_{k-1}}{A_{k+1}}\\
 & \leq q A_{k}+\sqrt{2+\rho}\frac{g_{k}^{2}}{A_{k+1}}-\sqrt{2+\rho}\frac{\rho g_{k}m_{k-1}}{A_{k+1}}.
\end{align*}
In the other case, where  $-g_{k}m_{k-1}<0$, then we can instead use $A_{k+1}\geq A_{k}$ and thus $-\frac{1}{A_{k}}\leq-\frac{1}{A_{k+1}}$ (given that $v_k$ is assumed positive),
so:
\begin{align*}
q A_{k}+\frac{g_{k}^{2}}{A_{k+1}}-\frac{\rho g_{k}m_{k-1}}{A_{k}} & \leq q A_{k}+\frac{g_{k}^{2}}{A_{k+1}}-\frac{\rho g_{k}m_{k-1}}{A_{k+1}}\\
 & \leq q A_{k}+\sqrt{2+\rho}\left[\frac{g_{k}^{2}}{A_{k+1}}-\frac{\rho g_{k}m_{k-1}}{A_{k+1}}\right].
\end{align*}
Using the notation $A_{k}=\gamma_{k+1}^{-1}\sqrt{b_{k}-\alpha_{k}}$,
where $b_{k}=\sum_{t=0}^{k}\left(g_{t}^{2}-\rho g_{t}m_{t-1}\right)$
and $\alpha_{k}=g_{k}^{2}-\rho g_{k}m_{k-1}$. We have from the concavity
of the square-root function that:
\begin{align*}
\sqrt{b_{k}-\alpha_{k}} & \leq\sqrt{b_{k}}+\frac{b_{k}-\alpha_{k}-b_{k}}{2\sqrt{b_{k}}}\\
 & =\sqrt{b_{k}}-\frac{\alpha_{k}}{2\sqrt{b_{k}}}.
\end{align*}
Therefore:
\begin{align*}
& qA_{k}+\sqrt{2+\rho}\left[\frac{g_{k}^{2}}{A_{k+1}}-\frac{\rho g_{k}m_{k-1}}{A_{k+1}}\right] \\
& \leq qA_{k+1}-\frac{q}{2\gamma_{k+1}}\left[\frac{g_{k}^{2}}{\sqrt{b_{k}}}-\frac{\rho g_{k}m_{k-1}}{\sqrt{b_{k}}}\right]+\sqrt{2+\rho}\left[\frac{g_{k}^{2}}{A_{k+1}}-\frac{\rho g_{k}m_{k-1}}{A_{k+1}}\right]\\
 & =qA_{k+1}-\frac{q}{2\gamma_{k+1}^{2}}\left[\frac{g_{k}^{2}}{A_{k+1}}-\frac{\rho g_{k}m_{k-1}}{A_{k+1}}\right]+\sqrt{2+\rho}\left[\frac{g_{k}^{2}}{A_{k+1}}-\frac{\rho g_{k}m_{k-1}}{A_{k+1}}\right]\\
 & =qA_{k+1}+\left(-\frac{2\sqrt{2+\rho}\gamma_{k+1}^{2}}{2\gamma_{k+1}^{2}}+\sqrt{2+\rho}\right)\left[\frac{g_{k}^{2}}{A_{k+1}}-\frac{\rho g_{k}m_{k-1}}{A_{k+1}}\right]\\
 & = qA_{k+1}.
\end{align*}

\end{proof}
\begin{lem}
\label{lem:gerror_bound}Let $q=2\sqrt{2+\rho}\gamma_{n}^{2}.$ Then:
\[
\sum_{k=0}^{n}\left[\left\Vert g_{k}\right\Vert _{A_{k+1}^{-1}}^{2}-\rho\left\langle g_{k-1},g_{k}\right\rangle _{A_{k}^{-1}}\right]\leq q\sum_{i=1}^{d}A_{i,1}+\frac{q}{\gamma_{0}}\sum_{i=1}^{d}\sqrt{\alpha_{i,n+1}}.
\]
\end{lem}
\begin{proof}
Without loss of generality, we again consider the 1D case. Recall that if $v_{k}<0$ (which can only occur when $k\geq 1$):
\begin{align*}
q A_{k}+\frac{g_{k}^{2}}{A_{k+1}}-\frac{\rho g_{k}m_{k-1}}{A_{k}} & \leq q A_{k}\\
 & =q A_{k+1}+q\left(A_{k}-A_{k+1}\right),
\end{align*}
and if $v_{k}>0$ and $k\geq 1$:
\[
q A_{k}+\frac{g_{k}^{2}}{A_{k+1}}-\frac{\rho g_{k}m_{k-1}}{A_{k}}\leq q A_{k+1}.
\]
We will telescope this bound from the base-case $\frac{g_{0}^{2}}{A_{1}}\leq qA_{1}$ (which is trivially true) to $A_{n+1}$.
Combining we have:
\begin{align*}
\sum_{k=0}^{n}\left[\frac{g_{k}^{2}}{A_{k+1}}-\frac{\rho g_{k}m_{k-1}}{A_{k}}\right] & \leq qA_{n+1}+q\sum_{k=1}^{n}I\left[v_{k}\leq0\right]\left(A_{k}-A_{k+1}\right)\\
 & =qA_{1}-q\sum_{k=1}^{n}\left(A_{k}-A_{k+1}\right)+q\sum_{k=1}^{n}I\left[v_{k}\leq0\right]\left(A_{k}-A_{k+1}\right)\\
 & =qA_{1}-q\sum_{k=1}^{n}I\left[v_{k}\geq0\right]\left(A_{k}-A_{k+1}\right)\\
 & =qA_{1}+q\sum_{k=1}^{n}I\left[v_{k}>0\right]\left(A_{k+1}-A_{k}\right)\\
 & =qA_{1}+q\sum_{k=1}^{n}I\left[v_{k}>0\right]\frac{1}{\gamma_{k}}\left(\sqrt{\alpha_{k+1}}-\sqrt{\alpha_{k}}\right)\\
 & \leq qA_{1}+\frac{q}{\gamma_{0}}\sum_{k=1}^{n}I\left[v_{k}>0\right]\left(\sqrt{\alpha_{k+1}}-\sqrt{\alpha_{k}}\right)\\
 & =qA_{1}+\frac{q\sqrt{\alpha_{n+1}}}{\gamma_{0}}.
\end{align*}

\end{proof}

\subsection{Proof of Theorem \ref{thm:main}}
Recall the general form of Theorem \ref{thm:telescoped-bound} from Section \ref{sec:thm-tele-proof}:
\begin{align*}2\mathbb{E}\sum_{k=0}^{n}\left[f(x_{k})-f_{*}\right] & \leq\mathbb{E}\left\Vert z_{0}-x_{*}\right\Vert _{A_{1}}^{2}+\left(\frac{2\beta}{1-\beta}+\rho\right)\left[f(x_{0})-f_{*}\right]\\
 & +\mathbb{E}\sum_{k=1}^{n}\left\Vert z_{k}-x_{*}\right\Vert _{\left(A_{k+1}-A_{k}\right)}^{2}\\
 & +\mathbb{E}\sum_{k=0}^{n}\left[\left\Vert g_{k}\right\Vert _{A_{k+1}^{-1}}^{2}-\rho\left\langle g_{k},m_{k-1}\right\rangle _{A_{k}^{-1}}\right].
\end{align*}

We apply the iterate error bound (Lemma \ref{lem:iterate-bound-lemma})
and the gradient error bound (Lemma \ref{lem:gerror_bound}) to give:
\begin{align*}2\mathbb{E}\sum_{k=0}^{n}\left[f(x_{k})-f_{*}\right] & \leq\left(\frac{2\beta}{1-\beta}+\rho\right)\left[f(x_{0})-f_{*}\right]\\
 & +\frac{D_{\infty}^{2}}{\gamma_{0}}\mathbb{E}\sum_{i=1}^{d}\sqrt{\alpha_{i,n+1}}+q\mathbb{E}\sum_{i=1}^{d}A_{i,1}+\mathbb{E}\frac{q}{\gamma_{0}}\sum_{i=1}^{d}\sqrt{\alpha_{i,n+1}}.
\end{align*}
Now using $f(\bar{x}_{n+1})-f_{*}\leq\frac{1}{n+1}\sum_{k=0}^{n}\left[f(x_{k})-f_{*}\right]$,
we have:
\begin{align*}\mathbb{E}\left[f(\bar{x}_{n+1})-f_{*}\right] & \leq\frac{1}{2(n+1)}\left(\frac{2\beta}{1-\beta}+\rho\right)\left[f(x_{0})-f_{*}\right]\\
 & +\left(\frac{D_{\infty}^{2}}{2\gamma_{0}(n+1)}+\frac{q}{2\gamma_{0}(n+1)}\right)\mathbb{E}\sum_{i=1}^{d}\sqrt{\alpha_{i,n+1}}+\frac{q}{2(n+1)}\mathbb{E}\sum_{i=1}^{d}A_{i,1}.
\end{align*}
Now using that $\alpha_{i,n+1}\leq(n+1)G_{\infty}^2,$ and $A_{i1}\leq G_{\infty}/\gamma_{0}$
and the definition of $q$:
\begin{align*}
\mathbb{E}\left[f(\bar{x}_{n+1})-f_{*}\right] & \leq\frac{1}{2(n+1)}\left(\frac{2\beta}{1-\beta}+\rho\right)\left[f(x_{0})-f_{*}\right]\\
 & +\left(\frac{D_{\infty}^{2}}{2\gamma_{0}(n+1)}+\frac{q}{2\gamma_{0}(n+1)}\right)d\sqrt{n+1}G_{\infty}+\frac{qd}{2(n+1)\gamma_{0}}G_{\infty}.
\end{align*}
Simplifying:
\begin{align*}
\mathbb{E}\left[f(\bar{x}_{n+1})-f_{*}\right] & \leq\frac{1}{2(n+1)}\left(\frac{2\beta}{1-\beta}+\rho\right)\left[f(x_{0})-f_{*}\right]\\
 & +\left(\frac{D_{\infty}^{2}}{2}+q\right)\frac{dG_{\infty}}{\gamma_{0}\sqrt{n+1}}.
\end{align*}
Now further using that $q=2\sqrt{2+\rho}\gamma_{n}^{2}\leq2\sqrt{2+\rho}D_{\infty}^{2}$:
\begin{align*}
\mathbb{E}\left[f(\bar{x}_{n+1})-f_{*}\right] & \leq\frac{1}{2(n+1)}\left(\frac{2\beta}{1-\beta}+\rho\right)\left[f(x_{0})-f_{*}\right]\\
 & +3\sqrt{2+\rho}\frac{dG_{\infty}D_{\infty}^{2}}{\gamma_{0}\sqrt{n+1}}.
\end{align*}

\section{Further experimental details}
For the convex experiments, we trained in PyTorch using CPU training. Each training run took several minutes at most, depending on the dataset.
Our DLRM training was run on a V100 GPU, and took approximately 6 hours per training run.
\end{document}